%% file: main.tex
\RequirePackage{fix-cm}
\documentclass{article}
\usepackage{iclr2027_conference,times}
\usepackage[hypertexnames=false]{hyperref}
\input{preamble}

\title{Who Bears the Burden? Learning\\Responsibility for Shared Constraints\\in Multi-Agent Reinforcement Learning}
\author{%
\setlength{\tabcolsep}{0pt}%
\begin{tabular}{@{}p{0.5\textwidth}p{0.5\textwidth}@{}}
\rule{0pt}{20pt}Xiaoyang Cao\thanks{Corresponding author.} & Jingqi Li \\
\mdseries Massachusetts Institute of Technology & \mdseries The University of Texas at Austin \\
\mdseries\texttt{xycao@mit.edu} & \mdseries\texttt{jingqi.li@austin.utexas.edu} \\[8pt]
Zhe Fu & Alexandre M. Bayen \\
\mdseries Stanford University & \mdseries University of California, Berkeley \\
\mdseries\texttt{zhefu@stanford.edu} & \mdseries\texttt{bayen@berkeley.edu}
\end{tabular}}
\iclrfinalcopy
\begin{document}
\maketitle
\lhead{Preprint}
\vspace{-1.5em}
\begin{center}
\textbf{Project page with code:} \href{https://lira-marl.github.io/}{\texttt{https://lira-marl.github.io}}
\end{center}
\input{sections/abstract}
\input{sections/introduction}
\input{sections/related_work}
\input{sections/method}

\input{sections/experiments}
\FloatBarrier
\input{sections/conclusion}
\input{sections/limitations}
\label{main-text-end}
\input{sections/reproducibility_statement}
\bibliographystyle{iclr2027_conference}
\bibliography{references}
\clearpage
\appendix
\input{appendix/a_equilibrium_selection}
\input{appendix/b_exact_examples}
\input{appendix/c_finite_gradient}
\input{appendix/e_implementation}
\input{appendix/f_experiments}
\end{document}

%% file: preamble.tex
\usepackage{amsmath,amssymb,amsthm,bm}
\usepackage{graphicx,booktabs,array,colortbl}
\usepackage{placeins,needspace}
\usepackage{pgfplots}
\pgfplotsset{compat=1.18}
\usepackage{algorithm,algpseudocode}
\usepackage{xurl}
\usepackage{etoolbox}%

\AtBeginDocument{%
  \abovedisplayskip=9pt plus 3pt minus 4pt
  \belowdisplayskip=9pt plus 3pt minus 4pt
  \abovedisplayshortskip=4pt plus 3pt
  \belowdisplayshortskip=6pt plus 3pt minus 3pt
}
\AtBeginEnvironment{algorithm}{\vspace{-2pt}}
\AtEndEnvironment{algorithm}{\vspace{-2pt}}
\usepackage{hyperref}
\definecolor{linkblue}{rgb}{0.0,0.25,0.55}
\hypersetup{colorlinks=true,linkcolor=linkblue,citecolor=linkblue,urlcolor=linkblue,filecolor=linkblue,
  pdfauthor={Xiaoyang Cao, Jingqi Li, Zhe Fu, Alexandre M. Bayen},
  pdftitle={Who Bears the Burden? Learning Responsibility for Shared Constraints in Multi-Agent Reinforcement Learning},
  pdfsubject={Lagrangian Responsibility Allocation (LiRA) for shared constraints in multi-agent reinforcement learning},
  pdfkeywords={multi-agent reinforcement learning, shared constraints, Lagrangian methods, generalized Nash equilibrium, responsibility allocation, social welfare}}
\newtheorem{theorem}{Theorem}
\newtheorem{proposition}{Proposition}

\theoremstyle{definition}

\theoremstyle{remark}

\newcommand{\E}{\mathbb E}
\newcommand{\R}{\mathbb R}
\newcommand{\calD}{\mathcal D}
\newcommand{\calA}{\mathcal A}
\newcommand{\calG}{\mathcal G}

\newcommand{\calU}{\mathcal U}
\newcommand{\dd}{\mathrm d}

\newcommand{\col}{\operatorname{col}}

\newcommand{\clip}{\operatorname{clip}}
\newcommand{\norm}[1]{\left\lVert#1\right\rVert}
\newcommand{\Dfix}{\mathrm D_{\phi}^{\mathrm{fix}}}
\newcommand{\DU}{\mathrm{DU}}
\newcommand{\SC}{\mathrm{SC}}

%% file: sections/abstract.tex
\begin{abstract}

When multiple agents share a cost budget, a common Lagrange multiplier can enforce the aggregate constraint but does not determine how its penalty should be allocated across agents. Uniform penalties ignore heterogeneity in the rewards agents sacrifice, while agent-specific multipliers may still rely on the same aggregate cost signal. We introduce \emph{\textbf{L}agrang\textbf{i}an \textbf{R}esponsibility \textbf{A}llocation} (LiRA), which learns each agent's share of a common multiplier by optimizing social welfare over a finite training horizon. The multiplier enforces the aggregate budget, while responsibility shares redistribute its influence without modifying the original rewards or constraints. For convex games under standard regularity conditions, varying these shares induces a smooth family of normalized generalized Nash equilibria in which active constraints remain at their budgets while welfare varies. To optimize responsibility before convergence, we derive a welfare gradient that accounts for both learning updates and the induced change in data distribution. Across CityLearn, MABIM, Harvest, and MetaDrive, spanning 3 to 400 agents, LiRA improves average social welfare by up to 29\% over uniform and agent-specific multiplier baselines. Grid and driving costs remain within budget, inventory violations decrease, and Harvest makes more effective use of available budget.

\end{abstract}

%% file: sections/introduction.tex
\section{Introduction}\label{sec:introduction}

\input{sections/new_intro}

\input{sections/figures/teaser}
\FloatBarrier

We test LiRA on four multi-agent tasks, including building energy
management in CityLearn~\citep{nweye2024citylearn} with 3 agents,
inventory control in MABIM~\citep{yang2023mabim} with 400 agents,
resource harvesting in Harvest (Melting Pot \nolinkurl{commons_harvest__open}~\citep{agapiou2022melting}) with 7 agents,
and autonomous driving in MetaDrive~\citep{li2022metadrive} with 4 agents.
LiRA has the highest mean welfare in all four tasks
(Section~\ref{sec:main-results}). 
It keeps grid costs within budget in CityLearn, reduces safety violations in MetaDrive, and cuts order-rejection penalties in MABIM.
A matched CityLearn episode (Section~\ref{sec:allocation-stability})
traces the gain to individual buildings: the two that overcool receive
a larger share of the grid-electricity penalty, cool less, and recover
comfort while district electricity falls.

%% file: sections/new_intro.tex
Consider two delivery robots approaching a narrow corridor. One carries an urgent package, while the other is traveling to collect a package. Either robot may yield to avoid a collision, but delaying the loaded robot incurs a larger loss. A shared collision-avoidance constraint excludes unsafe outcomes, but does not determine how the burden of satisfying the constraint should be allocated between the robots. This ambiguity arises more broadly in multi-agent systems subject to shared constraints, including power limits, resource budgets, and safety requirements, when agents incur heterogeneous costs for adapting their behavior.

A shared Lagrange multiplier specifies the strength of the penalty associated with a common constraint, but does not determine how that penalty should be apportioned across agents. Applying the same penalty to all agents leaves this allocation implicit. Introducing agent-specific multipliers does not resolve the issue when each multiplier is driven by the same aggregate constraint cost. We instead make the allocation of constraint responsibility explicit. Specifically, a responsibility share determines the fraction of the shared penalty applied to each agent's objective, while the common multiplier continues to enforce the aggregate constraint. The original rewards and constraint budget remain unchanged. Responsibility allocation therefore modifies which agents are incentivized to adapt without altering the underlying constrained problem.

This additional degree of freedom induces a nontrivial equilibrium-selection problem. Under convexity, monotonicity, and standard regularity conditions, varying the responsibility shares generates a smooth family of normalized generalized Nash equilibria. Along this family, each active constraint remains satisfied at its budget, while social welfare may vary (Theorem~\ref{thm:selection}). Motivated by this observation, we introduce \emph{\textbf{L}agrang\textbf{i}an \textbf{R}esponsibility \textbf{A}llocation} (LiRA), which optimizes responsibility shares to select among feasible equilibria according to social welfare.

LiRA does not require the underlying agents to converge to equilibrium before responsibility allocation can be optimized. Instead, it evaluates a candidate allocation through the welfare obtained after a finite lookahead of the learning dynamics. Changing the responsibility shares influences both the policy updates within this lookahead and the distribution of data collected for subsequent updates. We derive a gradient estimator that accounts for both effects by differentiating through the learning dynamics and correcting for the induced change in the sampling distribution (Theorem~\ref{thm:gradient}). Independent lookaheads initialized from a common training checkpoint allow this gradient to be estimated without modifying the primary learning trajectory. Figure~\ref{fig:teaser} summarizes the resulting framework, from responsibility allocation in the motivating corridor example to its implementation within the learning procedure.

%% file: sections/figures/teaser.tex
\begin{figure}[t]
\centering
\includegraphics[width=\linewidth]{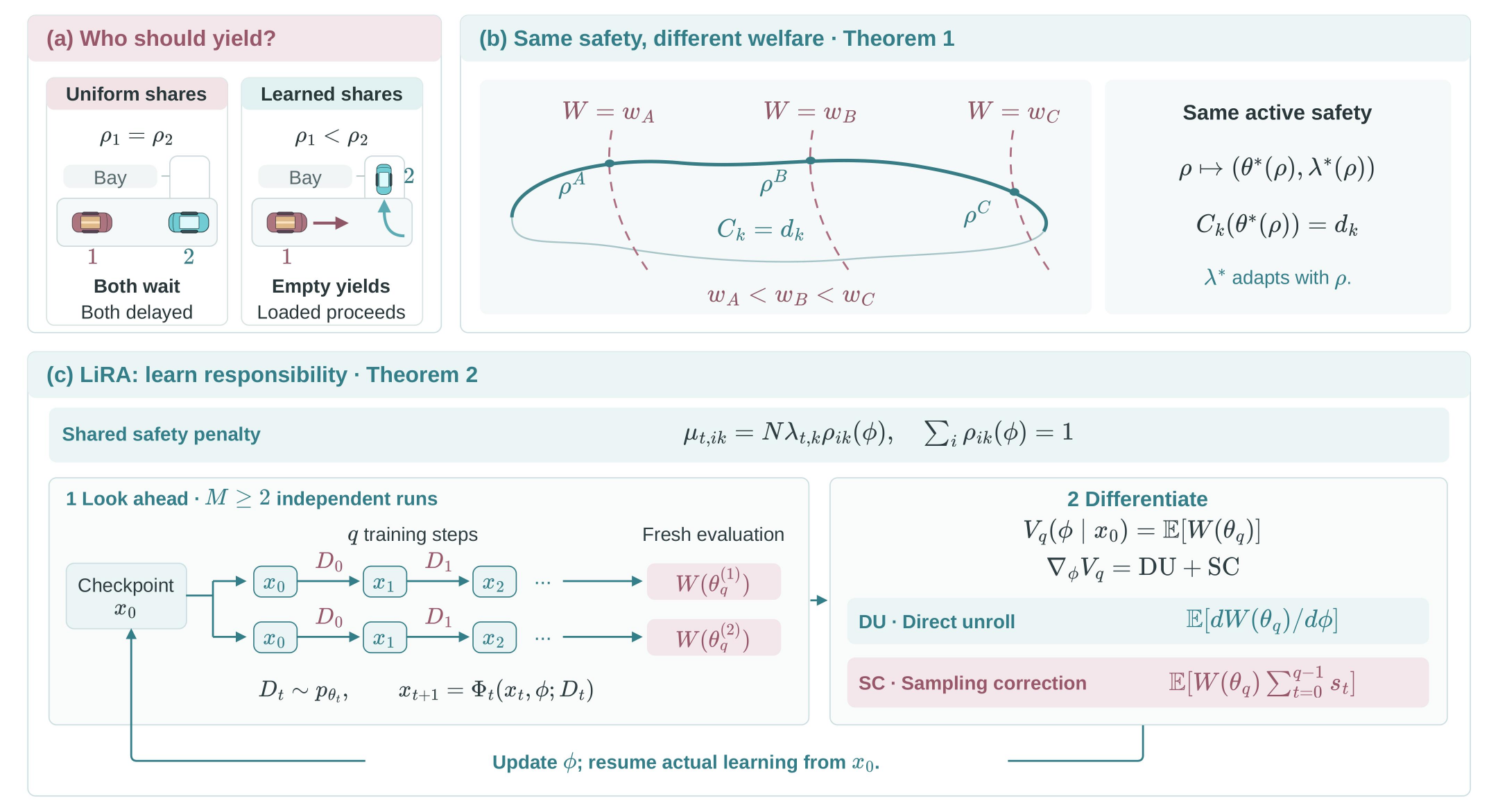}
\caption{\textbf{Overview of LiRA.}
\textbf{(a)} In the illustrated corridor scenario, uniform shares leave both
robots waiting; learned shares let the robot without an urgent package yield so that the loaded
robot proceeds. Both outcomes avoid collision but differ in welfare.
\textbf{(b)} Under Theorem~\ref{thm:selection}'s conditions, varying
responsibility $\rho$ selects a smooth equilibrium family. For each active
constraint $k$, cost $C_k$ stays at its budget $d_k$, while welfare $W$ can
vary and the shared equilibrium multiplier $\lambda^*$ adapts to $\rho$.
\textbf{(c)} From checkpoint $x_0$, LiRA runs $M\ge2$ independent lookaheads
of $q$ learner updates, each followed by fresh evaluation. It combines
direct unrolling (DU) with sampling correction (SC) to estimate $\nabla_\phi V_q$, update $\phi$, and restart from $x_0$.}
\label{fig:teaser}
\end{figure}

%% file: sections/related_work.tex
\section{Related Work}\label{sec:related-work}

\paragraph{Shared safety in multi-agent learning.}
Constrained reinforcement learning trades reward against explicit cost
budgets \citep{altman1999constrained}. Single-agent methods enforce
constraints via trust regions \citep{achiam2017constrained,yang2020projection}
or Lagrangian penalties \citep{tessler2019rcpo,stooke2020pid}. With
multiple agents, one learner's update shifts the feasible set for the
others. Constrained policy optimization and decentralized descent--ascent
address these coupled dynamics
\citep{gu2021macpo,lu2021decentralized}, and regret and violation bounds
exist for certain two-player settings \citep{ding2023provably}. All of
these methods enforce the shared cost, but none asks which agents
should absorb it. In multi-agent control, responsibility allocations
inferred from demonstrations through control barrier functions explain
how much each agent deviates to keep an interaction safe
\citep{remy2025responsibility}. LiRA keeps the shared cost feedback and
instead learns that distributional choice for social welfare during
reinforcement learning.

\paragraph{Equilibria under shared constraints.}
When agents share constraints, the feasible set of each depends on the
others' choices. Generalized Nash equilibria formalize such
interdependence \citep{facchinei2010generalized}. Rosen's normalized
equilibria fix the ratios between agents' constraint multipliers
\citep{rosen1965existence} and parameterized variational inequalities trace
the resulting family \citep{nabetani2011parametrized,dreves2018select}.
In dynamic games with shared constraints, non-normalized equilibria can
also be computed and a preferred one selected by a design criterion
\citep{pustilnik2025generalized}.
Recent safe-learning work studies convergence to such equilibria under
state-wise constraints \citep{li2024safe}. Our responsibility shares
set the multiplier ratios directly. Welfare then selects
among the equilibria they index.

\paragraph{Learning through the agents' response.}
Varying allocations can improve welfare at equilibrium,
but agents are still learning when shares must be chosen. Incentive design
modifies rewards in anticipation of agents' adaptive response
\citep{yang2020incentivize,yang2022adaptive,zheng2022economist}.
Bilevel formulations and differentiable equilibrium blocks offer
complementary approaches
\citep{thoma2024contextual,shen2024penalty,thoma2026deep}. LiRA differs
in that it allocates the penalty of an existing constraint rather
than designing new rewards. Its gradient combines
meta-gradient differentiation through learner updates
\citep{xu2018metagradient,franceschi2017forward,grefenstette2019gimli}
with a sampling correction for the data collected by intermediate
policies \citep{vuorio2022bias}. Iterative and implicit hypergradients
connect finite-window responses to equilibrium sensitivity
\citep{grazzi2020iteration}.

%% file: sections/method.tex
\section{Method}\label{sec:method}
The Lagrangian penalty for a shared constraint has a scale, set by the
multiplier, and a distribution across agents, usually left at $1/N$.
LiRA adds a learnable responsibility vector that controls the
distribution as the multiplier continues to track aggregate cost.
We formalize this allocation, show that it selects among equilibria with
different welfare at the same active costs, and derive a gradient for
choosing shares from a finite training window.
\input{sections/method/01_formulation}
\input{sections/method/02_equilibrium_selection}
\input{sections/method/03_responsibility_gradient}

\input{sections/method/04_algorithm}

%% file: sections/method/01_formulation.tex
\subsection{Responsibility allocation}\label{sec:formulation}
Consider a Markov game with $N$ agents and $K$ shared constraints. Agent $i$
uses policy $\pi_{\theta_i}(a_i\mid o_i)$, where $a_i$ is the action and $o_i$ is the local observation. The joint policy factorizes
across agents and induces trajectory distribution $p_\theta$, where $\theta:=[\theta_i]_{i=1}^N$. On a
trajectory $\tau$ of length $H$, let $r_{i,h}$ denote agent $i$'s reward
and $c_{k,h}$ the shared cost for constraint $k$ at step $h$. With discount
factor $\gamma\in[0,1]$ and budget $d_k$, define
\begin{equation}
 J_i(\theta)=\E_{\tau\sim p_\theta}\!\left[\sum_{h=0}^{H-1}\gamma^h r_{i,h}\right],
 \qquad
 C_k(\theta)=\E_{\tau\sim p_\theta}\!\left[\sum_{h=0}^{H-1}\gamma^h c_{k,h}\right]\le d_k.
 \label{eq:objectives}
\end{equation}
We measure social welfare using the original rewards,
$W(\theta)=\sum_i\alpha_iJ_i(\theta)$, with fixed positive weights
$\alpha_i$.

LiRA introduces a responsibility allocation. Let
$\rho_{ik}(\phi)$ be agent $i$'s share of the penalty for constraint $k$,
with $\rho_{ik}\ge0$ and $\sum_i\rho_{ik}=1$. At learner update $t$, the
shared multiplier $\lambda_{t,k}\ge0$ sets the overall penalty scale. Agent
$i$ receives coefficient
$\mu_{t,ik}=N\lambda_{t,k}\rho_{ik}(\phi)$ and trains on
\begin{equation}
 r^{\mathrm{eff}}_{i,h;t}
 =r_{i,h}-\sum_{k=1}^K\mu_{t,ik}c_{k,h}.
 \label{eq:allocation}
\end{equation}
Equal shares, $\rho_{ik}=1/N$, recover the usual Lagrangian coefficient
$\mu_{t,ik}=\lambda_{t,k}$. A larger $\rho_{ik}$ makes constraint $k$
matter more to agent $i$'s update. It does not alter the reward, cost,
or budget used to evaluate the resulting policies.
We optimize unconstrained logits $\phi_{ik}$ and map them to these
nonnegative shares with a softmax. A small positive floor can keep
every agent exposed to the shared penalty. Appendix~\ref{app:allocation-implementation}
gives the exact map.

The multiplier retains the standard Lagrangian primal--dual update from
constrained reinforcement learning
\citep{altman1999constrained,tessler2019rcpo,stooke2020pid}. At a scheduled
dual update, let $\widehat C_{t,k}$ be the estimated cost on the same scale
as its budget $d_k$, let $\beta_\lambda>0$ be the dual step size, and let
$[z]_+=\max\{z,0\}$. The shared multiplier follows
\begin{equation}
 \lambda_{t+1,k}
 =[\lambda_{t,k}+\beta_\lambda(\widehat C_{t,k}-d_k)]_+.
 \label{eq:dual}
\end{equation}
Each multiplier sees aggregate cost, regardless of who receives its
penalty. The factor $N$ in $\mu_{t,ik}$ ensures that equal shares
recover the original coefficient, so varying $\rho$ redistributes
pressure rather than rescaling it uniformly. We next show that this
redistribution can change welfare even when an active constraint stays
at its budget.

%% file: sections/method/02_equilibrium_selection.tex
\subsection{Welfare varies along the cost surface}\label{sec:selection}
In Figure~\ref{fig:teaser}(a), either robot can yield. A feasible
outcome need not be the best one for their combined rewards. The
constraint says what is forbidden, not who should give way.
Generalized Nash equilibria formalize such constrained best responses
\citep{facchinei2010generalized}. Among them, a normalized
equilibrium fixes the ratios of agents' constraint multipliers
\citep{rosen1965existence}. The responsibility shares set exactly those
ratios, so changing $\rho$ selects a different equilibrium under the
same shared constraint.

To make this precise, consider a positive responsibility vector $\rho=(\rho_i)_i$, with
$\rho_{ik}=\rho_i$ for every constraint $k$. Let $\mathcal R$ be an interior region of the simplex,
$\theta^*(\rho)$ the selected normalized equilibrium, and
$\lambda^*(\rho)$ its shared multipliers. Let $\calA$ denote the active shared
constraints. The regularity assumptions are collected in
Appendix~\ref{app:selection-assumptions}.

\begin{samepage}
\begin{theorem}[Equilibrium selection at fixed active costs]\label{thm:selection}
Under assumptions (A1)--(A3) in Appendix~\ref{app:selection-assumptions},
each $\rho\in\mathcal R$ selects a unique normalized equilibrium, and
$\theta^*(\rho)$ and $\lambda^*(\rho)$ are $C^1$ in $\rho$. Every
active shared cost remains at its prescribed budget:
\begin{equation}
 C_k(\theta^*(\rho))=d_k
 \qquad (k\in\calA).
 \label{eq:same-safety}
\end{equation}
If welfare has a nonzero derivative along the simplex at $\rho$, then nearby
responsibility allocations attain different welfare at those same active
costs.
\end{theorem}
The proof is in Appendix~\ref{app:selection-proof}.
\end{samepage}

The mechanism is visible in the agents' first-order conditions.
A larger share raises the constraint's marginal price for one agent
relative to another. When their marginal rewards for yielding differ,
the joint policy shifts along the active cost surface. Total cost
stays at the budget because the shared multiplier compensates, but the
agents divide the adjustment differently
(Figure~\ref{fig:teaser}(b)). The theorem makes this precise through
the nonzero welfare derivative. It does not guarantee that every
direction improves welfare, only that some direction does whenever the
derivative is nonzero.

Appendix~\ref{app:independent} extends the result locally to separate
shares for each constraint. During training, however, the agents have
not settled at an equilibrium, and we need to choose an allocation
from the response they actually produce.

%% file: sections/method/03_responsibility_gradient.tex
\subsection{Choosing shares during training}\label{sec:gradient}
Theorem~\ref{thm:selection} identifies the opportunity at equilibrium.
To realize it during training, LiRA evaluates a proposed allocation by
the welfare it produces after a finite number of updates, making the
training process itself the object we differentiate.

Let $x_t$ collect everything needed to resume training at update $t$,
including the policies, critics, optimizer state, and shared multipliers.
Let $\calD_t$ be the batch collected at that update, and let $\Phi_t$ denote
one differentiable lookahead step. Then
\begin{equation}
 \calD_t\sim p_{\theta_t}(\calD),\qquad
 x_{t+1}=\Phi_t(x_t,\phi;\calD_t).
 \label{eq:learner}
\end{equation}
The lookahead map may update the actors, critics, optimizer state, and
shared multipliers according to the chosen probe. It need only be
differentiable along a sampled training window. Our experiments use
proximal policy optimization (PPO)
\citep{schulman2017ppo}.

Starting from checkpoint $x_0$, let $q$ be the lookahead horizon. Holding the
allocation fixed for these $q$ updates gives the finite-training objective
\begin{equation}
 \max_\phi V_q(\phi\mid x_0),\qquad
 V_q(\phi\mid x_0)=\E_{\calD_{0:q-1}}[W(\theta_q)].
 \label{eq:meta}
\end{equation}
The objective measures welfare after the next $q$ updates under the
proposed shares, using the agents' original rewards.

Changing a share changes the update computed from a fixed training batch. The updated
policy then collects a different batch for the next update. To
account for this second path, define the training-batch score
\begin{equation}
 s_t=\frac{\dd}{\dd\phi}\log p_{\theta_t}(\calD_t),
 \label{eq:training-score}
\end{equation}
where the total derivative follows the recorded learner updates for a fixed
realized training history.

The distinction first appears after an update has changed the policy.
With $x_0$ fixed, the initial batch is collected before the proposed
allocation has changed any policy. A second batch can reflect the
allocation through the first update. A two-update lookahead is
therefore the shortest window that exposes both the direct update
and the changed-data path. Longer windows follow more of the ensuing
response but also require more sampled updates.

\begin{samepage}
\begin{theorem}[Finite-training responsibility gradient]\label{thm:gradient}
Suppose $x_0$ is fixed and the learner and sampling laws satisfy the
regularity conditions in Appendix~\ref{app:gradient}. Then
\begin{equation}
 \nabla_\phi V_q
 =\underbrace{\E\!\left[\frac{\dd W(\theta_q)}{\dd\phi}\right]}_{g^{\DU}}
 +\underbrace{\E\!\left[W(\theta_q)\sum_{t=0}^{q-1}s_t\right]}_{g^{\SC}}.
 \label{eq:gradient}
\end{equation}
\end{theorem}
The proof is in Appendix~\ref{app:gradient}.
\end{samepage}

The direct-unrolling (DU) term differentiates the recorded updates
at fixed data. Its complement, the sampling-correction (SC)
term, accounts for the data collected by intermediate policies
\citep{yang2020incentivize,yang2022adaptive,vuorio2022bias}.
Policy scores let us estimate the two terms without a simulator derivative.
Appendix~\ref{app:gradient} gives the derivation.

%% file: sections/method/04_algorithm.tex
\subsection{Estimating the LiRA update}\label{sec:algorithm}
The gradient in Eq.~\eqref{eq:gradient} can be estimated without
committing any proposed update to the live learner. From its current
checkpoint, LiRA runs $M\ge2$ independent continuations, each for $q$
updates at the same allocation. For continuation $m$, let
$\widehat W_m$ be its welfare on fresh terminal rollouts,
$\widehat g_m^{\DU}$ the direct-response estimate, and
$s^{(m)}=\sum_{t=0}^{q-1}s_t^{(m)}$ the score of its training history.

A lookahead of $q$ updates sets the response used to choose responsibility.
After that, $u$ live updates advance the learner before the next choice.
Update schedules for the two phases can differ. Appendix~\ref{app:experimental-protocol}
specifies the CityLearn probe and live schedules.

For its sampling term, continuation $m$ uses the other continuations'
mean welfare as a baseline,
\begin{equation}
 b_m=\frac{1}{M-1}\sum_{j\ne m}\widehat W_j.
 \label{eq:baseline}
\end{equation}
This gives
\begin{equation}
 \widehat g=\frac1M\sum_{m=1}^M
 \left[\widehat g_m^{\DU}+(\widehat W_m-b_m)s^{(m)}\right].
 \label{eq:loo}
\end{equation}
Because $b_m$ is independent of the score from continuation $m$, centering
preserves the expected gradient (Appendix~\ref{app:loo}). For an exact
differentiable lookahead, Eq.~\eqref{eq:loo} is unbiased. In CityLearn
we use the stopped-batch approximation described in
Appendix~\ref{app:lookahead-differentiates}, while holding the shared
multiplier fixed over the short probe to isolate responsibility
redistribution at the current dual scale.

Terminal welfare is estimated from the original rewards on fresh
rollouts. Its direct gradient follows the final policy's score back
through the recorded learner updates. The history score accumulates
the policy scores of batches collected during the lookahead. Reusing a
training batch for multiple PPO epochs does not create new sampled
actions in that history. These distinctions keep the two terms in
Eq.~\eqref{eq:loo} tied to the same finite response.

\begin{algorithm}[H]
\caption{LiRA: learning responsibility from finite continuations}
\label{alg:lira}
\begin{algorithmic}[1]
\Require Checkpoint $x_0$, allocation logits $\phi_0$, outer cycles $L$,
lookahead $q$, independent continuations $M\ge2$, live updates per cycle $u$,
step size $\eta_\rho$.
\For{$\ell=0,\ldots,L-1$}
  \For{$m=1,\ldots,M$}
    \State Restore $x^{(m)}_0\gets x_\ell$ and set $s^{(m)}\gets0$.
    \For{$t=0,\ldots,q-1$}
      \State Collect a fresh batch $\calD_t^{(m)}$ with the policy in $x_t^{(m)}$.
      \State $s^{(m)}\gets s^{(m)}+s_t^{(m)}$, the batch score in
        Eq.~\eqref{eq:training-score}.
      \State $x_{t+1}^{(m)}\gets
        \Phi_t(x_t^{(m)},\phi_\ell;\calD_t^{(m)})$.
    \EndFor
    \State On fresh terminal rollouts, estimate $\widehat W_m$ and
      $\widehat g_m^{\DU}$.
  \EndFor
  \State Compute $b_m$ for each $m$ and $\widehat g$ by Eq.~\eqref{eq:loo}.
  \State Ascend $\phi_\ell$ along $\widehat g$ with step size $\eta_\rho$
    to obtain $\phi_{\ell+1}$.
  \State Restore $x_\ell$; train the live learner for $u$ fresh updates
    under $\rho(\phi_{\ell+1})$, obtaining $x_{\ell+1}$.
\EndFor
\State \textbf{return} the final policy in $x_L$ and $\rho(\phi_L)$.
\end{algorithmic}
\end{algorithm}

The continuations are discarded after the allocation step. Only the
fresh live updates produce the next checkpoint, including its optimizer
and dual state. One outer cycle thus takes $Mq$ lookahead learner
updates and $u$ live updates, plus terminal rollouts. Appendix~\ref{app:implementation}
specifies the lookahead implementation.

%% file: sections/experiments.tex
\section{Experiments}\label{sec:experiments}
The theory above assumes convex games with smooth payoffs. We now test
whether LiRA produces welfare gains when agents train with nonlinear
function approximators for a finite time. Four tasks test welfare and
cost across different scales and constraint types. CityLearn then lets
us inspect individual agents. Throughout, we read costs alongside
rewards to distinguish gains obtained by spending unused budget from gains at
nearly the same shared cost.

\subsection{Welfare and shared costs across tasks}\label{sec:main-results}
The shared constraint plays a different role in each of our four tasks.
In CityLearn v2~\citep{nweye2024citylearn}, three buildings control their
energy devices for thermal comfort. They draw electricity from the same
district, so their combined net demand faces one grid cap: reducing
demand in one building can relieve a cost borne by all three. In
MABIM~\citep{yang2023mabim}, 400 product agents manage replenishment across an inventory network to maximize total profit, subject to shared order-rejection caps across two warehouse-store links.
An allocation now distributes pressure across hundreds of decisions
and two distinct constraints. Harvest, the Melting Pot~\citep{agapiou2022melting} substrate
\nolinkurl{commons_harvest__open} (built on Lab2D~\citep{beattie2020lab2d}), has seven agents collecting a common
resource under one shared cost. MetaDrive~\citep{li2022metadrive} has
four driving agents and a shared limit on the rate of safety events.
These tasks differ in scale and in what an agent gives up to reduce
cost.
In all four benchmarks, we set $\alpha_i=1$ for every agent, so welfare
is the sum of agents' original returns.

We compare LiRA with two baselines. Uniform responsibility
sets $\rho_{ik}=1/N$, and each agent receives
an equal share of each shared multiplier. The per-agent Lagrangian
baseline (PAL) instead learns a separate multiplier for each agent.
LiRA keeps a shared multiplier and learns its responsibility shares.
Within each task, methods use the same inner learner and matched training
seeds under a common evaluation protocol.
Table~\ref{tab:main-results} gives held-out means and
sample standard deviations over three training seeds, or six for
MetaDrive. Rewards and costs retain their native units: welfare is
compared within each task, and each cost against its own budget.
MABIM's rejection budgets are calibrated on validation data, with
eight matched action seeds in testing. MetaDrive starts each method
from the same actor checkpoint and uses 13 held-out episode seeds.
Appendix~\ref{app:experimental-protocol} gives further details of the protocol.

\input{sections/tables/experimental_results}

LiRA's mean welfare is highest in every domain
(Table~\ref{tab:main-results}), but the cost response differs in each.

\textbf{CityLearn.}
Uniform and PAL both return $-16.87$k in mean comfort reward. PAL
starts with equal multipliers and receives identical district-cost
feedback, so its multipliers stay equal. LiRA raises the welfare to
$-14.51$k, and the mean grid exceedance moves from $28.44$ to $28.73$
kWh, against a budget of $28.76$~kWh. The gain comes at nearly the same
district cost. Section~\ref{sec:allocation-stability} asks what the
buildings do differently to obtain it.

\textbf{MABIM.} With 400 agents and two constraints, LiRA reaches
$-591.36$ million versus Uniform's $-596.31$ million. It also
reduces the more strained rejection cost from $21.51$ to $21.30$
million. Both still exceed the budget of $20.35$~million, so the welfare
gain accompanies a cost improvement rather than a cost trade.

\textbf{Harvest.} The shared cost is well within budget: $33$--$37$
versus $68$. The welfare rises from $42.73$ (Uniform)
to $50.27$ using the unused allowance.

\textbf{MetaDrive.} LiRA reaches
$64.26$ in welfare (Uniform: $49.72$; PAL: $60.94$) and records
$0.890$ safety events per thousand steps, below both the $1.00$ limit
and Uniform's $1.017$. PAL improves welfare, but at $1.337$ events.
In MetaDrive, learning shares of a common penalty outperforms individually learned
penalties in both mean welfare and safety.

\pagebreak[3]
\subsection{What the buildings do differently}\label{sec:allocation-stability}
\input{sections/figures/citylearn_mechanism}
\input{sections/tables/citylearn_mechanism}
Each CityLearn building earns a thermal-comfort reward. The shared
cost accumulates the positive excess of district net electricity over a
per-step grid cap across $719$ steps, with a budget of $28.76$~kWh.
For an overcooled building, easing cooling serves both objectives at
once. It saves electricity and improves comfort. For a building that is already
too warm, the comfort effect reverses.

Figure~\ref{fig:citylearn-mechanism} follows a matched 719-step
evaluation under Uniform and LiRA. It folds the temperatures into
hour-of-day means, while Table~\ref{tab:citylearn-mechanism} reports
the allocations and totals from those same episodes. The shaded
16--18h price peak gives context for the daily response. The shared grid-cap
cost is accumulated over the whole evaluation.

Under Uniform, Building~1 is below the comfort band for $314$ steps
and Building~2 for $232$. Building~3 is above it for $318$. The first
two buildings are overcooled, while the third tends to run warm.

LiRA moves Building~1's daytime temperatures toward the comfort band.
Building~2, coldest late in the day, also warms: at 18h, within the marked
peak, its mean is $1.17^\circ$C below the setpoint instead of
$1.46^\circ$C. Building~3 begins above the band, so reducing its cooling would not
improve comfort. The same district cost signal reaches buildings
with different thermal needs.

The mean cooling action drops from $0.197$
to $0.184$ in Building~1 and from $0.312$ to $0.287$ in Building~2.
Building~3 also cools less ($0.180$ to $0.169$), but starts from a
warmer state.

The learned shares put more responsibility on the overcooled
buildings: $\rho=0.395$ for Building~1 and $0.387$ for Building~2,
against $0.218$ for Building~3 and $1/3$ each under Uniform.

Episode electricity use drops by $37.6$ and $32.7$ kWh for Buildings~1
and 2, and their cold steps drop to $269$ and $164$. However, the comfort
gains differ: $4.43$k versus $0.48$k. Building~1 also has much higher baseline cold discomfort (586.1 versus 214.5~$^\circ$C-steps), so the reduction in cold discomfort more than offsets the small increase in warm discomfort. Building~3 uses $18.9$ kWh less, with almost no change in
comfort reward ($-0.03$k).

In this seed-1102 episode, these building-level changes increase welfare by
$4.88$k as district cap exceedance falls from $25.42$ to
$23.78$ kWh. (The three electricity reductions sum to $89.2$ kWh.
The smaller change in cap exceedance reflects that shared cost counts
only electricity above the cap at each step.) The welfare improvement in this matched episode arises because buildings for which reduced cooling improves comfort and also relieves the grid receive more responsibility and adjust their behavior accordingly.
Appendix~\ref{app:citylearn-mechanism} provides the reward breakdown.

\subsection{Sensitivity of the outer loop}
\label{sec:hyperparameter-ablation}\label{sec:estimator-evidence}
The shares are learned from a finite training response, so their
quality depends on the outer-loop design: how far ahead to look ($q$),
how many independent lookaheads to average ($M$), and how fast the
shares move ($\eta_\rho$). We vary each one at a time around the
CityLearn setting used in Table~\ref{tab:main-results}, keeping the
other two fixed, with three matched seeds and $u=10$ live PPO updates
per cycle.

\input{sections/tables/outer_loop_sensitivity}
\FloatBarrier

The center $(q,M,\eta_\rho)=(2,8,0.05)$ gives the highest mean welfare
among these local choices ($-14.51$k), with $28.73$ kWh of grid excess
against a $28.76$~kWh budget. Looking three or four steps ahead instead
of two yields $-18.13$k or $-19.03$k, without reducing grid excess.
Four or twelve independent lookaheads also give lower mean welfare than
eight, at similar mean costs.

The responsibility step matters most in this sweep. Leaving it at
$0.025$ gives $-20.67$k and $32.97$ kWh; doubling it to $0.10$ gives
$-15.78$k and $32.75$ kWh. Neither neighbor improves the center's
welfare--cost outcome. This setting produces the district-level
gain in Table~\ref{tab:main-results} and the building responses in
Section~\ref{sec:allocation-stability}.

%% file: sections/tables/experimental_results.tex
\begin{table}[htbp]
\caption{\textbf{Welfare and shared costs across four tasks.}
Held-out mean $\pm$ sample standard deviation over matched training seeds
(three per task; six for MetaDrive). Costs and budgets use the units
shown in each row. Bold marks the highest mean welfare.}
\label{tab:main-results}
\centering
\small
\setlength{\tabcolsep}{3.6pt}
\renewcommand{\arraystretch}{1.16}
\begin{tabular*}{\linewidth}{@{\extracolsep{\fill}}llrrrr@{}}
\toprule
Task $(N,K)$ & Outcome & Budget & Uniform & PAL & LiRA \\
\midrule
CityLearn $(3,1)$ & Welfare ($10^3$) & -- & $-16.87\pm3.04$ & $-16.87\pm3.04$ & $\mathbf{-14.51}\pm1.66$ \\
 & Grid excess (kWh) & $28.76$ & $28.44\pm5.97$ & $28.44\pm5.97$ & $28.73\pm7.48$ \\
\addlinespace[3pt]
MABIM $(400,2)$ & Welfare ($10^6$) & -- & $-596.31\pm5.00$ & $-597.27\pm3.76$ & $\mathbf{-591.36}\pm4.48$ \\
 & Rejections $C_1$ ($10^6$) & $20.35$ & $21.51\pm0.25$ & $21.56\pm0.23$ & $21.30\pm0.18$ \\
 & Rejections $C_2$ ($10^3$) & $32.69$ & $27.71\pm2.31$ & $27.54\pm2.21$ & $28.68\pm0.90$ \\
\addlinespace[3pt]
Harvest $(7,1)$ & Welfare & -- & $42.73\pm12.22$ & $44.73\pm3.79$ & $\mathbf{50.27}\pm2.53$ \\
 & Shared cost & $68.00$ & $33.24\pm10.08$ & $33.38\pm5.18$ & $36.82\pm2.40$ \\
\addlinespace[3pt]
MetaDrive $(4,1)$ & Welfare & -- & $49.72\pm22.20$ & $60.94\pm13.45$ & $\mathbf{64.26}\pm19.98$ \\
 & Events / $10^3$ steps & $1.00$ & $1.017\pm0.343$ & $1.337\pm1.466$ & $0.890\pm0.407$ \\
\bottomrule
\end{tabular*}
\end{table}

%% file: sections/figures/citylearn_mechanism.tex
\begin{figure}[t]
\centering
\includegraphics[width=\linewidth]{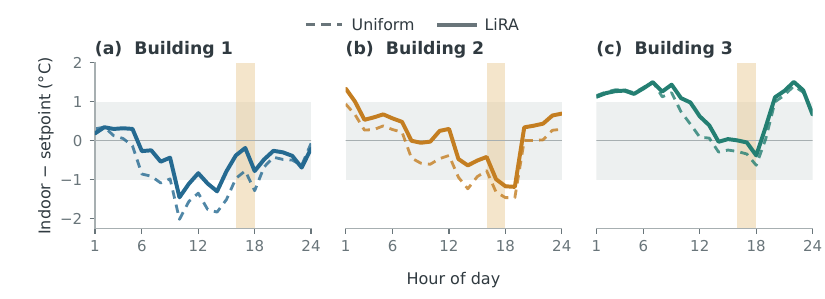}
\vspace{-0.8ex}
\caption{\textbf{Hourly temperature relative to setpoint in CityLearn.}
Each point averages observations at that hour in the matched 719-step
seed-1102 evaluation. Uniform is dashed; LiRA is solid. Gray marks the
$\pm1^\circ$C comfort band, and tan the 16--18h price peak. The
episode totals appear in Table~\ref{tab:citylearn-mechanism}.}
\label{fig:citylearn-mechanism}
\end{figure}

%% file: sections/tables/citylearn_mechanism.tex
\begin{table}[t]
\centering
\caption{\textbf{CityLearn responsibility and episode outcomes.}
Seed-1102 evaluation over 719 steps. Uniform assigns $\rho=1/3$
to each building; changes are LiRA minus Uniform.
Reward is native building comfort reward.}
\label{tab:citylearn-mechanism}
\vspace{0.3ex}
{\small
\begin{tabular*}{0.97\linewidth}{@{\extracolsep{\fill}}lrrr@{}}
\toprule
& Learned $\rho$ & Electricity change (kWh) & Reward change ($10^3$) \\
\midrule
Building 1 & $0.395$ & $-37.6$ & $+4.43$ \\
Building 2 & $0.387$ & $-32.7$ & $+0.48$ \\
Building 3 & $0.218$ & $-18.9$ & $-0.03$ \\
\bottomrule
\end{tabular*}
}
\end{table}

%% file: sections/tables/outer_loop_sensitivity.tex
\begin{table}[htbp]
\caption{\textbf{CityLearn local search around the main setting.}
Each row changes one parameter from $(q,M,\eta_\rho)=(2,8,0.05)$.
Welfare and grid excess are mean $\pm$ sample standard deviation over
three matched seeds. The grid-excess budget is $28.76$ kWh.}
\label{tab:outer-loop-sensitivity}
\centering
\small
\setlength{\tabcolsep}{4pt}
\renewcommand{\arraystretch}{1.14}
\begin{tabular*}{\linewidth}{@{\extracolsep{\fill}}llrr@{}}
\toprule
Varied choice & Setting & Welfare ($10^3$) & Grid excess (kWh) \\
\midrule
Main setting & $(2,8,0.05)$ & $\mathbf{-14.51}\pm1.66$ & $28.73\pm7.48$ \\
\addlinespace[2pt]
Lookahead & $q=3$ & $-18.13\pm4.42$ & $29.34\pm4.67$ \\
 & $q=4$ & $-19.03\pm8.39$ & $31.52\pm4.20$ \\
\addlinespace[2pt]
Replicates & $M=4$ & $-16.02\pm4.05$ & $28.91\pm4.02$ \\
 & $M=12$ & $-16.68\pm3.87$ & $29.12\pm2.38$ \\
\addlinespace[2pt]
Step size & $\eta_\rho=0.025$ & $-20.67\pm7.14$ & $32.97\pm5.65$ \\
 & $\eta_\rho=0.10$ & $-15.78\pm3.76$ & $32.75\pm12.50$ \\
\bottomrule
\end{tabular*}
\end{table}

%% file: sections/conclusion.tex
\section{Conclusion}\label{sec:conclusion}
A shared cost limit constrains what agents may do collectively but
leaves open who adjusts. LiRA makes this division learnable.
At equilibrium, responsibility shares can change welfare
even as active costs stay in their budgets. The finite-training gradient extends this
to nonconvergent learners by estimating welfare's response to a
proposed allocation through independent lookaheads. Across four tasks,
the learned shares give the highest mean welfare among the compared methods. In CityLearn, the welfare gain traces to individual
buildings: the two that overcool receive a larger penalty share, cool
less, and recover comfort while district electricity falls.

%% file: sections/limitations.tex
\section{Limitations}
Theorem~\ref{thm:selection} requires own-concave payoffs, convex
shared costs, and a strongly monotone weighted payoff field with a
fixed active-constraint set.  Deep policy networks violate these
conditions, and our finite-training gradient
(Theorem~\ref{thm:gradient}) does not guarantee welfare
improvement in such settings by itself.  Closing this gap, for instance, by
establishing local improvement guarantees for nonconvex learners
under shared constraints, is a substantial open problem.
The responsibility shares are also constant across states; every
state sees the same allocation.  A state-dependent
$\rho(s)$ could, for example, shift responsibility during
peak-demand hours, but it turns the outer problem from optimizing a
point on the simplex into learning a function, and the welfare
gradient would need to track how the allocation itself reshapes the
state distribution.

%% file: sections/reproducibility_statement.tex
\section*{Code Availability and Reproducibility}
Our code, per-seed result files, and scripts that regenerate the reported
tables are linked from the project page,
\url{https://lira-marl.github.io/}.
Benchmark settings, training seeds, and evaluation protocols are given in
Appendix~\ref{app:experiments}.

%% file: appendix/a_equilibrium_selection.tex
\section{Equilibrium selection}\label{app:selection}
We write the normalized constrained game as a variational inequality,
prove that a regular equilibrium branch moves smoothly with responsibility
while its active shared costs remain fixed, and extend the argument to
constraint-specific responsibility columns.
Two exact games in Appendix~\ref{app:examples} exhibit welfare variation
at fixed active cost and the benefit of separate constraint columns.

\subsection{Normalized game and assumptions}\label{app:selection-assumptions}
Agent $i$ chooses $\theta_i\in\Theta_i\subset\R^{n_i}$, where each
$\Theta_i$ is nonempty, compact, convex, and full-dimensional. Let
$\Theta=\prod_i\Theta_i$ and
\[
\mathcal K=\{\theta\in\Theta:C_k(\theta)\le d_k,\ k=1,\ldots,K\}.
\]
For this subsection, the same positive share $\rho_i$ is used for every
constraint. The effective multiplier in agent $i$'s objective is
$N\lambda_k\rho_i$. Dividing its stationarity condition by the positive
factor $N\rho_i$ leaves a common shared multiplier $\lambda_k$ across
agents. This gives the weighted negative-payoff field
\begin{equation}
 F_\rho(\theta)
 =\col_{i=1}^N\!\left[-\frac{1}{N\rho_i}\nabla_{\theta_i}J_i(\theta)\right].
 \label{eq:weighted-field}
\end{equation}
A normalized equilibrium is a solution of
\begin{equation}
 \langle F_\rho(\theta^*),\theta-\theta^*\rangle\ge0
 \qquad\text{for every }\theta\in\mathcal K.
 \label{eq:vi}
\end{equation}
This is Rosen's positive-weight normalization for a shared-constraint game
\citep{rosen1965existence}.

We use the following conditions.

\begin{enumerate}
\item[\textbf{(A1)}] Each $J_i$ and $C_k$ is $C^2$ on a neighborhood of
$\Theta$. For fixed $\theta_{-i}$, $J_i$ is concave in $\theta_i$, and each
$C_k$ is convex in $\theta$. There exists
$\bar\theta\in\operatorname{int}\Theta$ with
$C_k(\bar\theta)<d_k$ for every $k$.

\item[\textbf{(A2)}] Responsibility ranges over a connected, relatively open
set $\mathcal R$ in the simplex interior. For every $\rho\in\mathcal R$,
the weighted field is strongly monotone: there exists $m_\rho>0$ such that
\begin{equation}
 \langle F_\rho(\theta)-F_\rho(\vartheta),\theta-\vartheta\rangle
 \ge m_\rho\norm{\theta-\vartheta}^2
 \qquad(\theta,\vartheta\in\Theta).
 \label{eq:strong-monotonicity}
\end{equation}

\item[\textbf{(A3)}] On the region under consideration, the active shared
constraints and active local bounds are fixed. Active local bounds can be
written as owner-separable convex $C^2$ inequalities
$h_{i\ell}(\theta_i)\le0$. Their gradients together with those of the active
shared constraints are linearly independent. Active multipliers are positive,
and inactive inequalities remain strictly slack.
\end{enumerate}

The same conditions can be imposed only at a reference point, in which case
the conclusions below hold on a sufficiently small neighborhood.

\subsection{Proof of Theorem~\ref{thm:selection}}\label{app:selection-proof}

\paragraph{Existence and uniqueness.}
The feasible set $\mathcal K$ is nonempty, compact, and convex, and
$F_\rho$ is continuous. For any $\kappa>0$, the projected map
\[
\theta\mapsto\Pi_{\mathcal K}(\theta-\kappa F_\rho(\theta))
\]
maps $\mathcal K$ continuously into itself. Brouwer's theorem therefore gives
a fixed point, and the projection optimality condition is exactly
Eq.~\eqref{eq:vi}. If $\theta^*$ and $\vartheta^*$ are two solutions, adding
their two variational inequalities yields
\[
\langle F_\rho(\theta^*)-F_\rho(\vartheta^*),
        \theta^*-\vartheta^*\rangle\le0.
\]
Strong monotonicity in Eq.~\eqref{eq:strong-monotonicity} forces
$\theta^*=\vartheta^*$. Hence each $\rho$ selects a unique normalized policy.

\paragraph{Smooth dependence on responsibility.}
Let $\calA$ index the active shared constraints and $\mathcal L$ the active
local bounds. Let $G_C$ contain the active shared-cost gradients as rows and
$G_h$ the active local-bound gradients. With local multipliers
$\nu_{\mathcal L}$ and
$z=(\theta,\lambda_{\calA},\nu_{\mathcal L})$, define the active KKT residual
\begin{equation}
 \calG(z,\rho)=
 \begin{bmatrix}
 F_\rho(\theta)+G_C^\top\lambda_{\calA}+G_h^\top\nu_{\mathcal L}\\
 C_{\calA}(\theta)-d_{\calA}\\
 h_{\mathcal L}(\theta)
 \end{bmatrix}=0.
 \label{eq:weighted-kkt}
\end{equation}
Because agent $i$'s weighted stationarity block is scaled by
$1/(N\rho_i)$, Eq.~\eqref{eq:weighted-kkt} is equivalent to its original
stationarity condition with shared multiplier $N\lambda_k\rho_i$.
Own-concavity and convex feasible best-response sets then make these KKT
conditions sufficient for a constrained best response, so the selected
policy is a generalized Nash equilibrium.

Let $Q$ be the derivative in $\theta$ of the first row of
Eq.~\eqref{eq:weighted-kkt}, and let $G=\col(G_C,G_h)$. The Jacobian in the
primal-dual variables is
\begin{equation}
 \partial_z\calG=
 \begin{bmatrix}Q&G^\top\\G&0\end{bmatrix},\qquad
 Q=\partial_\theta F_\rho+
 \sum_{k\in\calA}\lambda_k\nabla^2 C_k+
 \sum_{\ell\in\mathcal L}\nu_\ell\nabla^2 h_\ell.
 \label{eq:bordered}
\end{equation}
The symmetric part of $Q$ is positive definite by strong monotonicity,
convexity, and nonnegative multipliers. Suppose $(u,v)$ lies in the kernel:
$Qu+G^\top v=0$ and $Gu=0$. Then $u^\top Qu=0$ forces $u=0$, and linear
independence of the active gradients gives $v=0$. The Jacobian is nonsingular.

The implicit function theorem gives a locally unique $C^1$ branch
$z^*(\rho)$. Positive active multipliers and strict slackness persist by
continuity. Uniqueness of the normalized policy means overlapping local
branches agree, giving a smooth selection on the full region.

\paragraph{Active costs and welfare.}
For every active shared constraint $k$,
\[
\lambda_k^*(C_k(\theta^*)-d_k)=0.
\]
Assumption (A3) keeps $\lambda_k^*>0$, so
$C_k(\theta^*(\rho))=d_k$ throughout the branch. If welfare has a nonzero
intrinsic derivative with respect to responsibility, then moving a small
distance in the improving direction changes welfare while all active costs
remain at those same budgets. This proves the remaining claims of
Theorem~\ref{thm:selection}.

\subsection{Constraint-specific responsibility columns}\label{app:independent}
The main theorem uses one responsibility vector for all constraints. LiRA
allows each shared constraint to have its own column. Around a regular
primal-dual point, the same argument applies to these independent columns.

Let $\psi$ be nonredundant local coordinates for the product of allocation
simplexes, and write $\rho_{ik}(\psi)$ for the responsibility of agent $i$
for constraint $k$. Define agent $i$'s stationarity field
\begin{equation}
 F_i(\theta,\lambda,\psi)=
 -\nabla_{\theta_i}J_i(\theta)+
 N\sum_k\lambda_k\rho_{ik}(\psi)\nabla_{\theta_i}C_k(\theta).
 \label{eq:independent-field}
\end{equation}
For active local multipliers $\xi$, the residual contains
$F+G_h^\top\xi$, $C_{\calA}-d_{\calA}$, and $h_{\mathcal L}$. Let $H_F$ be
the derivative of its first block in $\theta$, and let $A_R$ be the
derivative in $\lambda_{\calA}$; its $(i,k)$ block is
$N\rho_{ik}\nabla_{\theta_i}C_k$. The corresponding Jacobian is
\begin{equation}
 \mathbf J_{\rm ind}=\begin{bmatrix}
 H_F&A_R&G_h^\top\\
 G_C&0&0\\
 G_h&0&0
 \end{bmatrix}.
 \label{eq:independent-kkt}
\end{equation}

\begin{proposition}[Local independent-column selection]\label{prop:independent}
Suppose a reference primal-dual point satisfies this residual, primal and
dual feasibility, and strict complementarity. Assume the smoothness and
convexity conditions of (A1), smooth owner-local constraints, and
nonsingularity of $\mathbf J_{\rm ind}$. Then a locally unique $C^1$ primal-dual branch exists
for nearby $\psi$. Its policies are generalized Nash equilibria, its active
costs equal the same budgets, and nonzero welfare sensitivity gives different
welfare at those same active costs.
\end{proposition}

\begin{proof}
Apply the implicit function theorem to the full residual. Continuity preserves
multiplier signs and inactive slackness, while own-concavity makes the
resulting KKT points constrained best responses. Complementarity pins every
active cost at its budget, and the welfare derivative gives the final claim.
\end{proof}

If equality of the entire cost vector is required when some constraints are
inactive, let $Q_I(\psi)=C_I(\theta^*(\psi))$ collect those inactive costs.
When $\partial_\psi Q_I$ has locally constant rank, its level sets are smooth.
A tangent direction in $\ker(\partial_\psi Q_I)$ with nonzero welfare
derivative therefore generates a local curve along which the full cost vector
is fixed and welfare changes.

%% file: appendix/b_exact_examples.tex
\subsection{Exact examples}\label{app:examples}
Two small games make the equilibrium selection result concrete. The first
shows welfare varying along an active-cost level set. The second shows
why different constraints may prefer different responsibility allocations.

\subsubsection{Welfare variation at fixed active cost}\label{app:welfare-example}
Let $N\ge2$ agents choose scalar actions $u_i\in\R$. Agent $i$ prefers action
$b_i$, and the shared budget is $d>0$. Let $\alpha_i>0$ be the fixed welfare
weight for agent $i$, set $B=\sum_i b_i$ and $S=\sqrt{2d}$, and assume
$B>S$. Define
\begin{equation}
 J_i(u)=-\tfrac12(u_i-b_i)^2,\qquad
 C(u)=\tfrac12\Big(\sum_i u_i\Big)^2\le d.
 \label{eq:scalar-example}
\end{equation}
Compact action intervals containing the response family below can be imposed
without affecting the solution. For any simplex allocation, the
positive-aggregate active equilibrium is
\begin{equation}
 \lambda^*=\frac{B-S}{NS}>0,\qquad
 u_i^*(\rho)=b_i-(B-S)\rho_i,\qquad
 C(u^*(\rho))=d.
 \label{eq:scalar-response}
\end{equation}
To see this, stationarity gives
$b_i-u_i=N\lambda\rho_i\sum_j u_j$. Summing across agents and using the
active constraint $\sum_i u_i=S$ gives Eq.~\eqref{eq:scalar-response}.
The weighted field is diagonal and strongly monotone, so this family
satisfies the assumptions of Theorem~\ref{thm:selection}.

Substituting the equilibrium into welfare gives
\begin{equation}
 W(u^*(\rho))
 =-\tfrac12(B-S)^2\sum_i\alpha_i\rho_i^2,
 \qquad
 \rho_i^*=\frac{\alpha_i^{-1}}{\sum_j\alpha_j^{-1}}.
 \label{eq:inverse-weight}
\end{equation}
The active cost is identical for every allocation, yet welfare is not.
Unequal welfare weights make the uniform allocation strictly suboptimal;
with equal weights, uniform is optimal but welfare still changes away from
it.

\subsubsection{Why separate constraint columns can matter}\label{app:tied-example}
Now consider two agents with two-dimensional actions
$u_1=(u_{11},u_{12})$ and $u_2=(u_{21},u_{22})$. Both prefer
$b=(1,1)$, and the shared constraints are
$C_k(u)=u_{1k}+u_{2k}\le1$ for $k=1,2$. Let
\begin{align}
 J_1(u)&=-\tfrac12\norm{u_1-b}^2+(1/4,-1/4)^\top u_2,\\
 J_2(u)&=-\tfrac12\norm{u_2-b}^2+(-1/4,1/4)^\top u_1,
 \label{eq:externality-example}
\end{align}
with welfare $W=J_1+J_2$. Let $p_k$ be agent 1's responsibility share for
constraint $k$; agent 2 receives $1-p_k$. The equilibrium is
\begin{equation}
 \lambda_1^*=\lambda_2^*=\tfrac12,\qquad
 u_1^*=(1-p_1,1-p_2),\qquad
 u_2^*=(p_1,p_2).
 \label{eq:two-column-response}
\end{equation}
The two stationarity equations in each coordinate first determine the shared
multiplier and then the individual actions. The losses are strongly convex
and the constraints are linear, so these points are global constrained best
responses. Both shared costs equal one for every $(p_1,p_2)\in[0,1]^2$.

At this equilibrium,
\begin{equation}
 W(p_1,p_2)
 =-\tfrac12\left[p_1^2+(1-p_1)^2+p_2^2+(1-p_2)^2\right]
  +\tfrac12(p_1-p_2).
 \label{eq:two-column-welfare}
\end{equation}
Its Hessian is $-2I_2$, so the unique optimum is
$(p_1,p_2)=(3/4,1/4)$ with $W=-3/8$. If the two responsibility columns are
tied, $p_1=p_2=p$, then
$W=-(p^2+(1-p)^2)\le-1/2$, with equality at $p=1/2$. Independent columns
therefore improve the best tied welfare by $1/8$ without changing either
shared cost. The two constraints prefer opposite allocations because the
reward externalities have opposite signs across coordinates.

%% file: appendix/c_finite_gradient.tex
\section{Finite-training responsibility gradient}\label{app:gradient}
Theorem~\ref{thm:gradient} differentiates welfare through a random sequence
of learner updates. The proof conditions on the sampled training history,
differentiates welfare through the recorded updates at fixed data, and then
differentiates the probability of that history. These two derivatives yield
the direct-unrolling and sampling-correction terms.

\subsection{Training-history distribution}
Let $\phi\in\calU\subset\R^p$ be the allocation parameter and let $x_0$ be
fixed during one outer update. Write
$\mathbf D=(\calD_0,\ldots,\calD_{q-1})$ for the complete sequence of
training batches. Given a realized history, the learner states
$x_t(\phi;\calD_{0:t-1})$ are deterministic. For fresh batches, the joint
sampling law is
\begin{equation}
 p_\phi(\mathbf D\mid x_0)
 =\prod_{t=0}^{q-1}
 p_{\theta_t(\phi;\calD_{0:t-1})}(\calD_t).
 \label{eq:joint-data}
\end{equation}
With a persistent collector, the same argument uses the corresponding
history-conditioned sampling factors.

To separate the two sources of dependence, let $\Dfix$ denote the total
derivative $\dd/\dd\phi$ in Eq.~\eqref{eq:gradient}, taken while the
realized raw trajectories are held fixed. Observations, actions, rewards,
costs, terminal flags, and
other exogenous samples are fixed; learner quantities derived from them are
not. For a scalar, $\Dfix$ is a column gradient. In particular,
\begin{equation}
 \Dfix W(\theta_q)
 =(\Dfix\theta_q)^\top\nabla_\theta W(\theta_q),\qquad
 s_t=(\Dfix\theta_t)^\top
 \nabla_{\theta_t}\log p_{\theta_t}(\calD_t).
 \label{eq:chain-rules}
\end{equation}

The proof assumes that the conditional sampling laws share
parameter-independent support, the recorded learner maps and welfare are
differentiable almost surely, and differentiation can be exchanged with
integration. For piecewise-smooth learner updates, these conditions are
understood on the realized differentiable branches.

\subsection{Proof of Theorem~\ref{thm:gradient}}
Using Eq.~\eqref{eq:joint-data},
\begin{align}
 \nabla_\phi V_q
 &=\nabla_\phi\int
   W(\theta_q(\phi;\mathbf D))\,p_\phi(\mathbf D\mid x_0)\,\dd\mathbf D\\
 &=\int\left[
   \Dfix W(\theta_q)
   +W(\theta_q)\Dfix\log p_\phi(\mathbf D\mid x_0)
   \right]p_\phi(\mathbf D\mid x_0)\,\dd\mathbf D\\
 &=\E\!\left[
   \Dfix W(\theta_q)
   +W(\theta_q)\sum_{t=0}^{q-1}s_t
   \right].
 \label{eq:gradient-proof}
\end{align}
The first term differentiates welfare through the learner for a fixed
training history. The second differentiates the probability of seeing that
history. This proves Eq.~\eqref{eq:gradient}.

\subsection{Policy-score form}\label{app:training-score}
The sampling correction can be evaluated without differentiating the
environment. Let $P$ denote the transition law and $p_0$ the reset law. For a
trajectory with states $\mathsf{z}_h$, local observations $o_{i,h}$, and joint actions
$a_h$,
\begin{equation}
 p_\theta(\tau)
 =p_0(\mathsf{z}_0)\prod_{h=0}^{H-1}
 \left[
 P(\mathsf{z}_{h+1}\mid\mathsf{z}_h,a_h)
 \prod_{i=1}^N\pi_{\theta_i}(a_{i,h}\mid o_{i,h})
 \right].
 \label{eq:trajectory-law}
\end{equation}
At a fixed realized trajectory, the environment factors do not depend on
$\phi$. The batch score therefore reduces to policy scores:
\begin{equation}
 s_t
 =\sum_{\tau\in\calD_t}\sum_{h\in\tau}\sum_{i=1}^N
 \Dfix\log\pi_{\theta_{t,i}}(a_{i,h}\mid o_{i,h}).
 \label{eq:joint-score}
\end{equation}
Each sampled action appears once, even if the same batch is reused for
multiple policy-optimization epochs.

\subsection{Terminal welfare estimate}\label{app:terminal}
For a terminal rollout $\tau$, let
\[
R_W(\tau)=\sum_{h=0}^{H-1}\gamma^h\sum_i\alpha_i r_{i,h}
\]
be its original social return. Given the training history, draw
$E\ge1$ fresh terminal rollouts $\tau_1,\ldots,\tau_E$ under
$\pi_{\theta_q}$ and define
\begin{equation}
 \widehat W
 =\frac1E\sum_{e=1}^E R_W(\tau_e),\qquad
 \widehat g^{\DU}
 =\frac1E\sum_{e=1}^E
 R_W(\tau_e)(\Dfix\theta_q)^\top
 \nabla_\theta\log p_{\theta_q}(\tau_e).
 \label{eq:terminal-estimator}
\end{equation}
Conditioned on the training history, these estimate
$W(\theta_q)$ and $\Dfix W(\theta_q)$, respectively. In a general-sum game,
the social return weights every agent's policy score, preserving the effect
of one agent's action on the others' rewards.

\subsection{Leave-one-out centering}\label{app:loo}
For independent continuations $m=1,\ldots,M$ from the same checkpoint, let
$s^{(m)}=\sum_{t=0}^{q-1}s_t^{(m)}$ and let $b_m$ be the baseline in
Eq.~\eqref{eq:loo}. A score has conditional mean zero:
\begin{equation}
 \E[s^{(m)}\mid x_0,\phi]
 =\int\nabla_\phi p_\phi(\mathbf D\mid x_0)\,\dd\mathbf D=0.
\end{equation}
Because $b_m$ is constructed only from the other independent continuations,
$\E[b_ms^{(m)}\mid x_0,\phi]=0$. Together with the conditional unbiasedness
of the terminal estimates,
\begin{equation}
 \E[\widehat g\mid x_0,\phi]
 =\nabla_\phi V_q(\phi\mid x_0).
 \label{eq:loo-unbiased}
\end{equation}
Thus leave-one-out centering changes neither term in expectation. The
continuation estimates need not be mutually independent after centering,
because their baselines share the other continuations' returns; only the
independence between a continuation's score and its own baseline is used.

%% file: appendix/e_implementation.tex
\section{Lookahead implementation}\label{app:implementation}

\subsection{PPO and critic updates}
Let $\vartheta_i$ denote agent $i$'s actor parameters within a PPO epoch and
$\theta_i^{\mathrm{beh}}$ the parameters used to collect the batch.
Let $\widehat A^R_{i,h}$ and $\widehat A^C_{k,h}$ be the reward and
shared-cost advantages. For a fixed responsibility allocation, the
effective advantage
$\widehat A^{\mathrm{eff}}_{i,h}$ and policy ratio $\zeta_{i,h}$ are
\begin{equation}
 \widehat A^{\mathrm{eff}}_{i,h}
 =\widehat A^R_{i,h}-\sum_k N\lambda_k\rho_{ik}(\phi)\widehat A^C_{k,h},\qquad
 \zeta_{i,h}(\vartheta_i)=
 \frac{\pi_{\vartheta_i}(a_{i,h}\mid o_{i,h})}
      {\pi_{\theta_i^{\mathrm{beh}}}(a_{i,h}\mid o_{i,h})}.
 \label{eq:advantages}
\end{equation}
Let $\epsilon_{\mathrm{PPO}}>0$ be the clipping radius,
$c_{\mathrm{ent}}\ge0$ the entropy coefficient, $\mathcal H$ policy entropy,
and $\widehat\E$ a masked average over valid samples. The clipped loss is
\begin{equation}
 \widehat{\mathcal L}_i(\vartheta_i)=
 -\widehat\E\!\left[\min\left\{
 \zeta_{i,h}(\vartheta_i)\widehat A^{\mathrm{eff}}_{i,h},
 \clip(\zeta_{i,h}(\vartheta_i),1-\epsilon_{\mathrm{PPO}},1+\epsilon_{\mathrm{PPO}})
 \widehat A^{\mathrm{eff}}_{i,h}\right\}\right]
 -c_{\mathrm{ent}}\widehat\E[\mathcal H(\pi_{\vartheta_i})].
 \label{eq:ppo}
\end{equation}
PPO clips the combined reward-and-cost advantage
\citep{schulman2017ppo}. In CityLearn, the reward advantages are the
per-building, per-step comfort rewards standardized within each batch;
the cost advantages are the district cap excess standardized over that
batch. The critics regress to the corresponding unnormalized per-step
sum of building rewards and district cap excess.

\subsection{What the lookahead differentiates}\label{app:lookahead-differentiates}
The CityLearn probe starts from copies of the live actors, critics, and
Adam moments. Each fresh batch gives one functional PPO actor--critic
update and one differentiable Adam step. The next batch is collected
under the updated probe policy. The shared multiplier is held at its
checkpoint value during this $q$-step probe to isolate responsibility
redistribution at the current dual scale; it is updated in the live
learner on the schedule specified in Appendix~\ref{app:experimental-protocol}.

Within a probe update, stored behavior log-probabilities, normalized
advantages, and critic targets are held fixed. Differentiation follows
the actor and critic updates and the carried Adam moments, while the
training-history score accounts for actions collected under the changing
probe policies. The PPO implementation therefore uses a stopped-batch
approximation to the exact direct derivative in
Theorem~\ref{thm:gradient}: its sampling correction does not restore
derivatives omitted from the update itself.

\subsection{Allocation and online updates}\label{app:allocation-implementation}
For logits $\phi_{ik}$ and a fixed floor $\varepsilon_\rho\in[0,1/N)$,
\begin{equation}
 \rho_{ik}(\phi)=\varepsilon_\rho+
 (1-N\varepsilon_\rho)\frac{\exp\phi_{ik}}{\sum_j\exp\phi_{jk}}.
 \label{eq:floor}
\end{equation}
The floor keeps shares away from zero; setting it to zero recovers softmax.
Adding a constant to a logit column leaves its allocation unchanged.
The equilibrium proofs use nonredundant local simplex coordinates, whereas
autodiff can operate directly on the logits. An ascent optimizer applies
Eq.~\eqref{eq:loo} to update $\phi$.

Algorithm~\ref{alg:lira} uses a lookahead from the current checkpoint.
Each of its $M$ independent continuations takes $q$ probe updates at fixed
responsibility, collecting fresh training batches and terminal rollouts.
After the outer step, the live learner takes $u$ updates from that
checkpoint under the new allocation, producing the full learner state $x^+$.
Held-out evaluation
measures its policy without changing that state. The next cycle uses
$x_{\ell+1}$ (written $x^+$ above) as its checkpoint and retains the updated $\phi$, carrying forward the
actors, critics, optimizer state, and shared multipliers from the live
continuation.

\subsection{Computation and evaluation}
The terminal rollouts used in the meta-gradient train $\phi$; separate
reporting rollouts evaluate the resulting policy. Appendix~\ref{app:experimental-protocol}
records benchmark settings and evaluation conventions for Table~\ref{tab:main-results}.
A responsibility-update cycle uses $Mq$ probe batches and $u$ live
batches, in addition to terminal and held-out evaluation rollouts.
The lookahead cost is part of the method's computation budget.

%% file: appendix/f_experiments.tex
\section{Experimental protocols and reward components}\label{app:experiments}

\subsection{Benchmark setup for Section~\ref{sec:main-results}}
\label{app:experimental-protocol}
We use the same training seed identifiers across methods within each task:
three seeds, except six for MetaDrive. We first summarize the held-out
evaluations within each seed; Table~\ref{tab:main-results} reports the mean
and sample standard deviation across these seed-level values. Costs and
welfare retain each benchmark's native units.

\paragraph{CityLearn.}
We use CityLearn v2 \citep{nweye2024citylearn} with the three buildings in
\nolinkurl{citylearn_challenge_2023_phase_1}. Welfare is the sum of native
ComfortReward across buildings and steps. Shared cost accumulates only the
positive excess of district net electricity over the per-step grid cap.
The $d=0.04$ experiments have a 719-step cap-exceedance budget of
$28.76$ kWh ($719\times0.04$) and use seeds 1101--1103, 30 outer cycles,
$u=10$ live PPO updates per cycle, 12-step live rollouts, four PPO epochs,
shared-dual initialization $2.0$, and dual learning rate $0.05$.
For each CityLearn lookahead, every probe step collects a fresh 12-step
batch and makes one functional PPO/Adam actor--critic update. The shared
multiplier stays fixed during the probe. A separate 12-step terminal
rollout supplies the mean per-step comfort reward for the allocation
update; its action scores and the training-batch scores are both divided
by the 12-step batch length. Before the outer optimizer, this scales
Eq.~\eqref{eq:loo}'s unnormalized score estimate by $1/12$.
After the allocation update, the live
learner instead takes $u=10$ fresh batches with four PPO epochs each,
then updates the shared multiplier once from their mean grid excess.
That dual update compares mean per-step excess against $0.04$ kWh.
The held-out values in Table~\ref{tab:main-results} are 719-step
episode totals, rather than the probe's short-rollout objective.
The LiRA reference point uses $(q,M,\eta_\rho)=(2,8,0.05)$ in both
Table~\ref{tab:main-results} and Table~\ref{tab:outer-loop-sensitivity}.
The local search changes one coordinate at a time on the same three training
seeds, keeping $u=10$ live updates per outer cycle.

In this setting, PAL starts with three equal per-building multipliers. Their
updates use the same district cost and budget, so they remain equal and
produce the same result as Uniform.

\paragraph{MABIM.}
The inventory experiment has $N=400$ controlled SKU agents and $K=2$ shared
rejection-cost coordinates: warehouse 0/store 2 and warehouse 1/store 1.
For Table~\ref{tab:main-results}, we use training seeds 1101--1103, $q=5$,
$M=6$ lookaheads, ten outer cycles (50 committed updates), and budgets
calibrated on a validation period separate from the held-out test period.
Held-out evaluation pairs the same eight action seeds across LiRA, Uniform,
and PAL.
Each training seed contributes the mean of these eight held-out rollouts to
Table~\ref{tab:main-results}; the displayed spread is across training seeds.

\paragraph{Harvest.}
The Harvest task uses seven agents on the Melting Pot~\citep{agapiou2022melting}
substrate \nolinkurl{commons_harvest__open}, built on
Lab2D~\citep{beattie2020lab2d}, and one episode-level shared cost with budget $68$.
The three training seeds are 10411, 10442, and 10473. The LiRA setting
uses 60 response updates, $q=10$, $M=4$, and responsibility step size
$0.15$, with an allocation update every five response updates. The shared
dual starts at $0.3$ and uses step size $0.05$. Each training seed is
evaluated on five held-out episodes with native stochastic resets.

\paragraph{MetaDrive.}
MetaDrive uses six matched training seeds (211--216), a shared dual
learning rate $\lambda_{\rm lr}=4.0$ across all three arms, and the active
training-time cost budget $d=0.001$ per actual environment step (shown in
Table~\ref{tab:main-results} as $d=1.00$ events per 1000 actual held-out
steps, the same threshold on that display scale). All three arms start
from the same competent fixed actor checkpoint and are evaluated on the
same 13 held-out episode seeds; welfare is the mean per-episode welfare
and cost is the number of native shared-cost events per 1000 steps
actually run.
Within each training seed, we average welfare over the 13 episodes and
divide the total event count by their total realized steps to obtain the
cost rate. Table~\ref{tab:main-results} then aggregates these six seed-level
values.

\subsection{CityLearn mechanism details}\label{app:citylearn-mechanism}

\begin{table}[h]
\caption{\textbf{CityLearn comfort components in the matched seed-1102 episode.}
Cold excess is measured under Uniform; reward changes are LiRA minus Uniform.}
\label{tab:citylearn-reward-components}
\centering
\begin{tabular}{lrrrr}
\toprule
& Cold excess & $\Delta R_{\rm cold}$ & $\Delta R_{\rm warm}$ &
$\Delta R_{\mathrm{in\text{-}band}}$\\
& ($^\circ$C$\cdot$steps) & & &\\
\midrule
Building 1 & 586.1 & +4528.01 & $-102.80$ & +6.55\\
Building 2 & 214.5 & +813.09 & $-340.49$ & +6.26\\
Building 3 & 70.2  & +159.58 & $-205.55$ & +11.62\\
\bottomrule
\end{tabular}
\end{table}

Figure~\ref{fig:citylearn-mechanism} uses seed 1102 with
$q=2$, $M=8$, and $\eta_\rho=0.05$. Comfort states are classified against
the effective $\pm1^\circ$C ComfortReward band.
Table~\ref{tab:citylearn-reward-components} separates the cold, warm, and
in-band contributions to the comfort changes in
Table~\ref{tab:citylearn-mechanism}.

Building~1 begins with the greatest cold burden, so its reduction in
cooling yields the largest comfort recovery. Across all three buildings,
episode welfare increases by $4876.26$ while district cap excess moves from
$25.4156$ to $23.7787$ kWh.